\documentclass[10pt,twocolumn,letterpaper]{article}

\usepackage{iccv}
\usepackage{times}
\usepackage{epsfig}
\usepackage{graphicx}
\usepackage{amsmath}
\usepackage{amssymb}
\usepackage{graphicx}
\usepackage[table]{xcolor}% 
\usepackage{algorithm,algpseudocode}
\usepackage[normalem]{ulem}
\usepackage[utf8]{inputenc} % allow utf-8 input
\usepackage[T1]{fontenc}    % use 8-bit T1 fonts
\usepackage{url}            % simple URL typesetting
\usepackage{booktabs}       % professional-quality tables
\usepackage{amsfonts}       % blackboard math symbols
\usepackage{nicefrac}       % compact symbols for 1/2, etc.
\usepackage{microtype}      % microtypography
\usepackage{xcolor}         % colors
\usepackage{enumitem}
\usepackage[english]{babel}
\usepackage{amsthm}
\usepackage{wrapfig}
\usepackage{subfig}  
\usepackage{multirow}
\makeatletter
\def\hlinewd#1{%
\noalign{\ifnum0=`}\fi\hrule \@height #1 \futurelet
\reserved@a\@xhline}
\makeatother

\newtheorem{theorem}{Theorem}
\newtheorem{proposition}{Proposition}

\usepackage[pagebackref=true,breaklinks=true,letterpaper=true,colorlinks,bookmarks=false]{hyperref}

\iccvfinalcopy % *** Uncomment this line for the final submission

\def\iccvPaperID{1884} % *** Enter the ICCV Paper ID here

\ificcvfinal\fi

\begin{document}

%%%%%%%%% TITLE
\title{
% Prompt Condensation: Towards Efficient Visual Prompt Tuning
Do We Really Need a Large Number of Visual Prompts?
}


% \and

\author{Youngeun Kim\\
Yale University\\
{\tt\small youngeun.kim@yale.edu}
\and
Yuhang Li\\
Yale University\\
{\tt\small yuhang.li@yale.edu}
\and
Abhishek Moitra \\
Yale University\\
{\tt\small abhishek.moitra@yale.edu}
\and
Ruokai Yin \\
Yale University\\
{\tt\small ruokai.yin@yale.edu}
\and
Priyadarshini Panda\\
Yale University\\
{\tt\small priya.panda@yale.edu}
}

\maketitle
\ificcvfinal\thispagestyle{empty}\fi

%%%%%%%%% ABSTRACT
\begin{abstract}
   Due to increasing interest in adapting models on resource-constrained edges, parameter-efficient transfer learning has been widely explored. 
   Among various methods, Visual Prompt Tuning (VPT), prepending learnable prompts to input space, shows competitive fine-tuning performance compared to training of full network parameters.
   However, VPT increases the number of input tokens, resulting in additional computational overhead.
   In this paper, we analyze the impact of the number of prompts on fine-tuning performance and self-attention operation in a vision transformer architecture. Through theoretical and empirical analysis we show that adding more prompts does not lead to linear performance improvement.
   Further, we propose a Prompt Condensation (PC) technique that aims to prevent performance degradation from using a small number of prompts. 
   We validate our methods on FGVC and VTAB-1k tasks and show that our approach reduces the number of prompts by $\sim$70\% while maintaining accuracy.
\end{abstract}

%%%%%%%%% BODY TEXT
\section{Introduction}

Parameter-Efficient Transfer Learning (PETL) has become a popular approach in various domains as it enables fine-tuning pre-trained models with minimal memory usage on resource-constrained edge devices \cite{rebuffi2018efficient, zhang2020side, zhang2021tip, zhou2022learning, he2022parameter,hu2021lora}. In PETL, a large model with billions of parameters, such as a transformer \cite{dosovitskiy2020image,vaswani2017attention}, is first trained on a massive dataset on a cloud server, and then fine-tuned with limited computational/memory resources on edge devices. 
Among various PETL methods, Visual Prompt Tuning (VPT) \cite{jia2022visual} is promising due to its ability to update a small subset of parameters while achieving higher accuracy than other methods.
Technically, VPT introduces learnable prompt tokens, which are prepended to the input or intermediate image patch tokens.

\begin{table}[t]
   \centering
\small
\resizebox{0.47\textwidth}{!}{%
\begin{tabular}{l|ccccc}
\hlinewd{1pt}
\# Prompts (ViT-B/16 \cite{dosovitskiy2020image}) & 0 & 50 & 100 &150 &200  \\
\hline
 GFLOPs & 17.6 & 22.2 & 26.9 & 31.8 & 36.7 \\
 Computational Overhead & 0\% & 26.1\% & 52.8\% & 80.6\%  & 108.5\% \\
 \hline
  \hline
 \# Prompts (Swin-B \cite{liu2021swin}) & 0 & 5 & 10 & 25 & 50  \\
 \hline
 GFLOPs & 15.4 & 16.3 & 17.2 & 19.8 & 24.3 \\
 Computational Overhead & 0\% & 5.8\% & 11.6\% & 28.5\% & 57.8\%\\
\hlinewd{1pt}
\end{tabular}%
}
\vspace{-2mm}
\caption{ The increase of floating-point operations (FLOPs) with respect to the number of prompts in VPT \cite{jia2022visual}. 
}
\label{table:intro:Flops_vs_prompt}
  \vspace{-1mm}
\end{table}

\begin{figure}[t]
\begin{center}
\centering
\def\arraystretch{0.5}
\begin{tabular}{@{\hskip -0.02\linewidth}c@{\hskip -0.01\linewidth}c@{\hskip -0.01\linewidth}c@{\hskip -0.01\linewidth}c@{}c@{}c@{}c}
% \hspace{-4mm}
\includegraphics[width=0.34\linewidth]{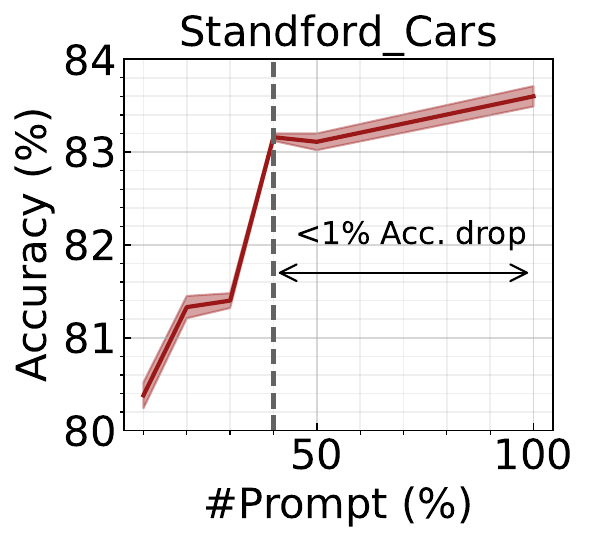} &
\includegraphics[width=0.34\linewidth]{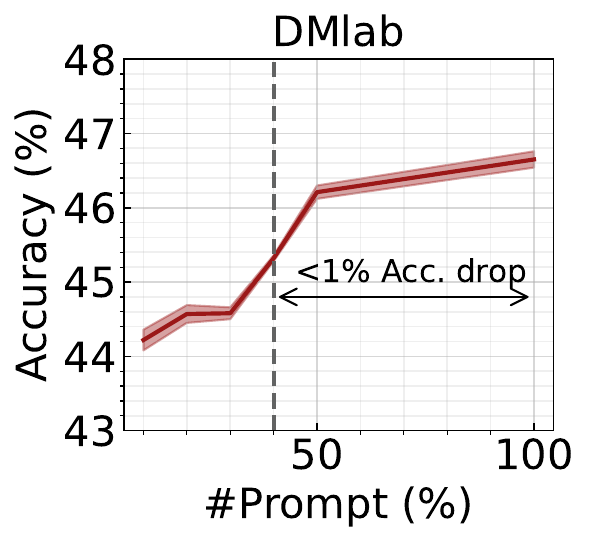} &
\includegraphics[width=0.34\linewidth]{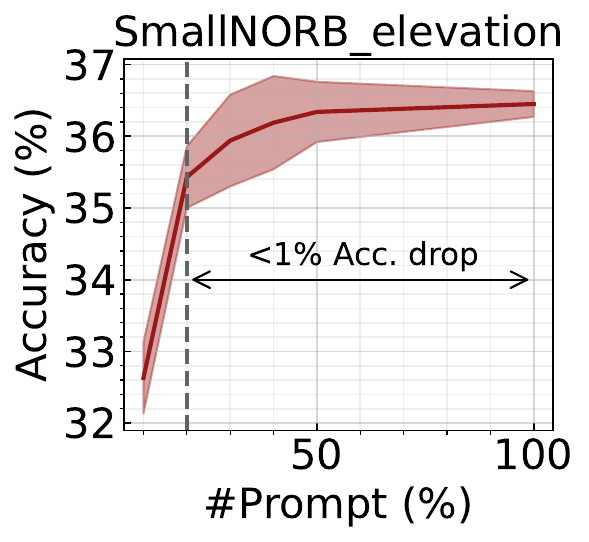} 
\end{tabular}
\end{center}
\vspace{-6mm}
\caption{ 
 Accuracy depending on the number of prompts  used for VPT training. We transfer an ImageNet-22k pre-trained ViT-B/16 \cite{dosovitskiy2020image} to three downstream tasks. The x-axis shows the relative number of prompts compared to the original number reported in \cite{jia2022visual}. The vertical dotted line indicates the point where there is < 1\% drop in accuracy from 100\% number of prompts. }
 \vspace{-1mm}
\label{fig:intro:performance}
\end{figure}

While VPT can induce memory efficiency, the use of additional prompt tokens leads to increased computational costs from self-attention and linear layers \cite{liu2021swin,yu2021unified,dosovitskiy2020image}. We report FLOPs with respect to the number of prompts in Table \ref{table:intro:Flops_vs_prompt}, which shows that the computational cost of VPT significantly increases as the number of prompts increases. If 200 prompts are prepended to the input space of ViT-B, the computational overhead (\ie FLOPs) almost doubles compared to the model without any prompts.
This indicates there is an inevitable trade-off between the number of prompts and computational cost in VPT.

Given such a trade-off, it is natural to ask: \textit{How does the fine-tuning performance change with respect to the number of prompts?}
To find the answer, we measure the test accuracy with respect to the number of prompts.
Interestingly, as shown in Fig. \ref{fig:intro:performance}, we found that reducing the number of prompts for VPT training by approximately 50\% does not lead to a significant drop, and most of the performance drop happens in the $10\% \sim 40\%$ range. 
The results imply that the correlation between the number of prompts and fine-tuning accuracy is not linear.

To further provide a better understanding of the prompts in VPT, we analyze the impact of the number of prompt tokens on fine-tuning accuracy by addressing several questions:
\textit{Why does the number of prompts and the fine-tuning performance have a non-linear correlation? How does the number of prompts affect the self-attention operation? If there is a performance drop with less number of prompts, how can we recover the accuracy drop?}
We provide both empirical and mathematical analysis to answer such questions.
This can provide insights into the behavior of the VPT model and its self-attention mechanism, which can help researchers better understand VPT and potentially improve the prompt design.
At the same time, it is essential to analyze this impact on the computational cost to ensure that the method remains practical for deployment on extremely resource-constrained edge devices.

A noteworthy observation from Fig. \ref{fig:intro:performance} is that the performance degradation in $<50\%$ number of prompts regime is non-trivial.
To address this, we propose \textit{Prompt Condensation} (PC), a technique that reduces the number of prompt tokens with minimal accuracy drop.
The PC consists of three steps: (1) Computing the importance score for each prompt. Here, we propose a global metric for measuring the importance score of each prompt, which provides better accuracy compared to the local attention-based metrics \cite{liang2022not,rao2021dynamicvit,fayyaz2022adaptive}.
(2) Selecting the top $k\%$ prompts based on the importance score, and discard the remaining prompts.
(3) Fine-tuning the selected prompts while freezing other parameters.

In summary, our contributions can be as follows: 
\begin{itemize}
\item In a first-of-its-kind study, we analyze the impact of the number of visual prompt tokens on the fine-tuning accuracy and self-attention operation in VPT.
\item We find that the number of prompts is not linearly proportional to performance improvement. To support this, we provide empirical and mathematical analysis.
\item To recover the performance drop with a small number of prompts, we propose \textit{Prompt Condensation} (PC).
Our method can reduce the number of prompts by $\sim 70\%$ while maintaining performance. 
\end{itemize}

\section{Related Work}

\subsection{Parameter Efficient Transfer Learning (PETL)}
Efficient fine-tuning of large pre-trained models on edge devices has become a popular research topic due to its practicality and high performance \cite{rebuffi2018efficient, zhang2020side, zhang2021tip, zhou2022learning, he2022parameter,hu2021lora}. Rather than training the entire set of parameters in neural networks, researchers focus on how to use a small percentage of weights to maximize transfer performance. To this end, several approaches \cite{rusu2016progressive,sung2022vl,houlsby2019parameter, cai2020tinytl} insert a lightweight bottleneck module into the transformer model, allowing gradients to be calculated only for a small number of parameters. TinyTL \cite{cai2020tinytl} and BitFit \cite{zaken2021bitfit} propose to update the bias term to fine-tune the model. Other approaches \cite{zhang2020side, sung2022lst} add side networks that can be optimized while keeping the original large model frozen.
Another effective method to reduce memory requirements is to sparsify \cite{jiangback} or quantize activation \cite{chakrabarti2019backprop,chen2021actnn,fu2020don,evans2021ac} during backward gradient calculation. 
Recently, VPT \cite{jia2022visual} prepends trainable parameters to the input space of the pre-trained model, achieving similar (and sometimes even better) accuracy compared to full fine-tuning while optimizing only about $1\%$ of the parameters. However, adding a large number of prompts can significantly increase the computational overhead of the model. In this work, we analyze how the number of prompts affects fine-tuning performance.

\noindent\textbf{{Importance of our work.}}  Prompt tuning is one of the major research directions to fine-tune the large-scale pre-trained model. 
Considering that prompt learning is applied to various applications, we aim to improve the efficiency of the prompt tuning approach.
 Our objective differentiates from prior works \cite{chen2022vision,bahng2022exploring,chen2022adaptformeR,zhang2022neural,jie2022convolutional}, such as adapter-based or partial training methods, which primarily seek to enhance performance on downstream tasks with different approaches.
 Furthermore, given that our technique does not necessitate any modifications to the model architecture, it offers promising potential for extension in future prompt learning approaches.

\subsection{Token Sparsification}

The computational cost of ViT \cite{dosovitskiy2020image} increases as the number of tokens given to the model increases \cite{tay2022efficient}. To alleviate this issue, previous works aim to reduce the number of tokens \cite{fayyaz2022adaptive,rao2021dynamicvit,meng2022adavit, kong2022spvit,goyal2020power,yu2021unified,kim2020length,kim2021learned,song2022cp}. Liang \etal \cite{liang2022not} define the importance score of each token based on its similarity to a $[CLS]$ token. Rao \etal \cite{rao2021dynamicvit} propose a prediction module with Gumbel-Softmax to sparsify tokens, which is jointly trained with the model parameters. Meng \etal \cite{meng2022adavit} propose a decision network that can turn on/off heads and blocks in a transformer architecture. The authors of \cite{yin2022vit} propose an adaptive halting module that calculates a probability for each token to determine when to halt processing. However, these methods require updating the weight parameters inside the transformer or an additional module, which is challenging to apply to the PETL scenario. Recently, \cite{bolya2022token} proposed a token merging technique without training, gradually reducing the number of tokens in each block of vision transformers to speed up inference. However, their method will be difficult to apply for prompt tokens because prompt tokens are introduced at every layer.

\begin{table*}[t]
   \centering
\small
\resizebox{0.99\textwidth}{!}{%
\begin{tabular}{c|cccccccc}
% \toprule
\hlinewd{1pt}
\#Prompts$\setminus$Data   & CUB-200  & NABirds  & Stanford Dog &  Stanford Car &  CIFAR100 &  SVHN  &  EuroSAT &  Resisc45   \\
\hline
 100\% & $88.52_{\pm 0.09}$& $84.20_{\pm 0.05}$ & $90.22_{\pm 0.09}$ & $83.42_{\pm 0.11}$ & $78.51_{\pm 0.71}$ & $80.64_{\pm 1.28}$ & $96.41_{\pm 0.39}$ & $82.66_{\pm 1.54}$ \\
 50\% & $88.45_{\pm 0.04}$ & $84.21_{\pm 0.06}$ & $90.25_{\pm 0.05}$ & $83.22_{\pm 0.09}$ & $78.23_{\pm 0.84}$&  $80.66_{\pm 0.39}$ & $96.09_{\pm 0.08}$& $82.18_{\pm 0.11}$ \\
  40\% & $88.45_{\pm 0.10}$& $84.18_{\pm 0.02}$ & $90.21_{\pm 0.06}$ & $83.16_{\pm 0.04}$ &$77.87_{\pm 0.49}$& $80.58_{\pm 0.42}$ & $95.88_{\pm 0.40}$& $82.10_{\pm 0.97}$\\
   30\% & $88.49_{\pm 0.10}$& $84.16_{\pm 0.04}$ & $90.22_{\pm 0.06}$ & $81.90_{\pm 0.08}$ &$78.18_{\pm 0.98}$& $78.49_{\pm 1.65}$ & $95.88_{\pm 0.40}$& $82.53_{\pm 0.81}$\\
    20\% & $88.47_{\pm 0.09}$& $84.11_{\pm 0.04}$ & $90.22_{\pm 0.09}$ & $81.42_{\pm 0.12}$ &$78.08_{\pm 0.68}$& $79.08_{\pm 1.46}$ & $95.98_{\pm 0.27}$& $82.36_{\pm 0.40}$ \\
     10\% & $88.13_{\pm 0.11}$ & $84.13_{\pm 0.03}$ & $90.20_{\pm 0.07}$ & $80.76_{\pm 0.14}$ &$77.62_{\pm 0.23}$& $77.56_{\pm 0.89}$ & $95.90_{\pm 0.12}$& $81.21_{\pm 0.57}$\\
\hline
\hline
\#Prompts$\setminus$Data  & Clevr/count  & Clevr/dist & DMLab & KITTI/dist &  dSprites/loc &  dSprites/ori &  SmallNORB/azi &  SmallNORB/ele    \\
\hline
 100\% & $68.65_{\pm 1.24}$ & $59.05_{\pm 0.32}$ & $46.05_{\pm 0.33}$ & $72.89_{\pm 2.20}$ &$74.35_{\pm 2.80}$ & $48.09_{\pm 1.77}$ & $32.86_{\pm 0.84}$ & $36.46_{\pm 0.19}$ \\
 50\% & $68.49_{\pm 2.12}$ & $59.68_{\pm 0.60}$ & $46.21_{\pm 0.87}$ & $72.26_{\pm 1.38}$ &$72.26_{\pm 3.11}$& $47.50_{\pm 1.36}$ & $32.43_{\pm 0.28}$& $36.34_{\pm 0.42}$ \\
  40\% & $68.88_{\pm 1.70}$ & $59.45_{\pm 0.38}$ & $45.32_{\pm 0.50}$ & $72.32_{\pm 1.40}$ &$69.02_{\pm 4.05}$& $47.18_{\pm 0.41}$ & $32.22_{\pm 0.43}$& $36.19_{\pm 0.65}$ \\
   30\% & $66.40_{\pm 0.83}$ & $58.94_{\pm 0.20}$ & $44.58_{\pm 1.10}$ & $72.33_{\pm 1.29}$ &$67.48_{\pm 5.75}$& $47.37_{\pm 1.17}$ & $31.24_{\pm 0.76}$& $35.94_{\pm 0.64}$ \\
    20\% & $65.62_{\pm 3.23}$ & $58.94_{\pm 0.57}$ & $44.57_{\pm 0.88}$ & $72.14_{\pm 1.12}$ &$58.14_{\pm 6.26}$& $47.22_{\pm 0.73}$ & $29.29_{\pm 1.86}$& $35.43_{\pm 0.43}$ \\
     10\% & $60.49_{\pm 3.08}$ & $58.83_{\pm 0.21}$ & $44.22_{\pm 0.89}$ & $72.39_{\pm 1.30}$ &$52.26_{\pm 6.06}$& $44.46_{\pm 2.57}$ & $29.03_{\pm 1.50}$& $32.63_{\pm 0.50}$\\
\hlinewd{1pt}
\end{tabular}%
}
\vspace{-1mm}
\caption{ The test accuracy of VPT-Deep on FGVC and VTAB-1k tasks with respect to the number of prompts. We report 6 different prompt settings, where $k\%$ represents how many prompts we use for VPT training compared to the original number of prompts reported in \cite{jia2022visual}.
We run the same configuration 3 times and report the mean and standard deviation.
We transfer an ImageNet pre-trained ViT-B/16 \cite{dosovitskiy2020image} backbone.
}
\label{table:analysis:acc_vs_prompt}
  \vspace{-2mm}
\end{table*}

\section{Preliminary}

\noindent\textbf{Vision Transformer.} 
Our work is based on ViT \cite{dosovitskiy2020image} which processes image tokens with multiple attention operations.
The input image is sliced into multiple patches (\ie tokens).
Then, in each layer, the self-attention operation is applied to the image tokens.
Let's assume we have token embedding $X \in \mathbb{R}^{n \times d}$,
Query $Q = XW_q$, Key $K = XW_k$, Value $V = XW_v$ with linear projection. Then, the attention operation can be formulated as follows.
\begin{equation}
    Attention(Q,K,V) = \underbrace{Softmax(\frac{QK^T}{\sqrt{d}})}_{A}V,
    \label{attention_eq}
\end{equation}
where $A$ is the self-attention matrix after the Softmax function.
We utilize Multi-Head Self-Attention (MHSA) in our approach, which takes the outputs of multiple single-head attention blocks, and then projects the combined output using an additional parameter matrix.
\begin{equation}
    head_i = Attention(XW_q^i,XW_k^i,XW_v^i).
\end{equation}
\begin{equation}
    MHSA(X) = Concat[head_1, ..., head_H]W_o + X.
\end{equation}
The output tokens generated by the MHSA block are fed into a Feed-Forward Network (FFN), which is composed of two fully-connected layers with a GELU activation layer in between. In the last encoder layer of the Transformer, the $[CLS]$ token is extracted from the output tokens and employed to predict the class.

\noindent\textbf{Visual Prompt Tuning.}
Visual Prompt Tuning (VPT) \cite{jia2022visual} suggests a memory-efficient fine-tuning technique by adding a set of learnable prompts at the input/intermediate layers. Depending on where the prompts are added, VPT has two versions: VPT-Shallow and VPT-Deep.

Let $X_i \in \mathbb{R}^{n \times d}$ be the token embedding at layer $i\in \{1,2,...,L\}$, and $F_i(\cdot)$ be the operations inside layer $i$. VPT-Shallow prepends $m$ prompts $P_1 \in \mathbb{R}^{m \times d}$ to the input token embedding $X_1$. 
\begin{equation}
    [Z_2, X_2] = F_1([P_1;X_1]).
\end{equation}
\begin{equation}
    [Z_{i+1}, X_{i+1}] = F_i([Z_i; X_i]) \hspace{4mm} \textup{for}  \hspace{2mm} 1<i\leq L.
\end{equation}
Here, $Z_{i}$ is the output tokens from the layer $i$.
Note that only $P_1$ and the classification head are trained. 

On the other hand, VPT-Deep introduces prompts $P_i \in \mathbb{R}^{m \times d}$ in every layer. 
\begin{equation}
    [\hspace{1mm} \_\_ \hspace{1mm}, X_{i+1}] = F_i([P_i;X_i]) \hspace{4mm} \textup{for}  \hspace{2mm} 1\leq i\leq L.
\end{equation}
VPT-Deep shows higher performance than VPT-Shallow by using more prompts. In our work, we focus on VPT-Deep due to its superior performance.
Although VPT requires significantly less memory usage for training, the computational overhead increases as the total number of tokens increases.

\begin{figure}[t]
\begin{center}
\centering
\def\arraystretch{0.5}
\begin{tabular}{@{\hskip 0.01\linewidth}c@{\hskip 0.03\linewidth}c@{}c}
% \hspace{-4mm}
\includegraphics[width=0.45\linewidth]{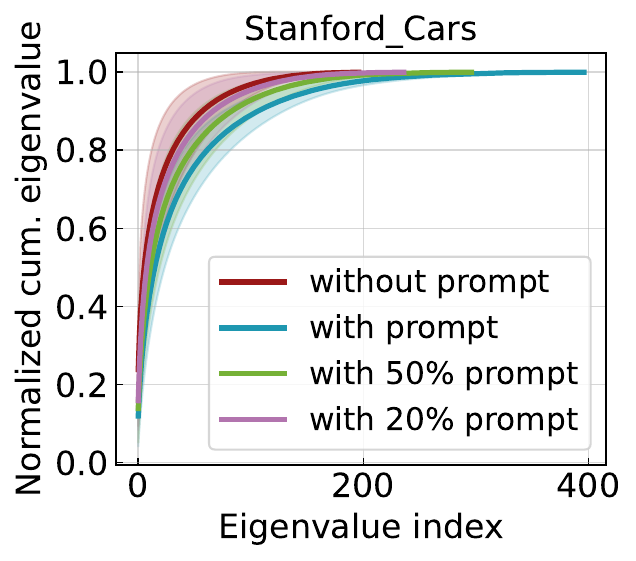} &
\includegraphics[width=0.45\linewidth]{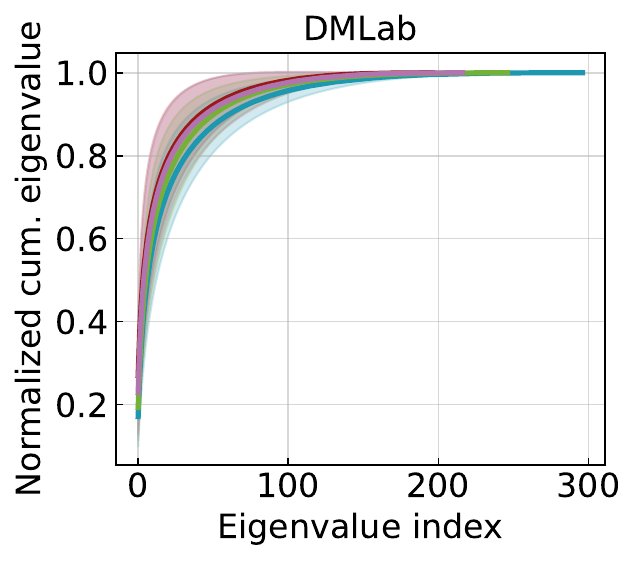} 
\end{tabular}
\end{center}
\vspace{-6mm}
\caption{ 
The normalized cumulative eigenvalue of self-attention matrix $A$ in Eq. \ref{attention_eq} on {Stanford Cars} and {DMLab}. We report the mean and standard deviation across all layers. 
% The $x$ range of is different depending on the number of prompts.
}
 \vspace{-2mm}
\label{fig:anlysis:performance}
\end{figure}

\section{Analysis on the Number of Visual Prompts}
\label{section:analysis}
In this section, we analyze the impact of prompts on self-attention operation and fine-tuning accuracy. We first demonstrate two observations, and then we provide the mathematical support why the performance does not linearly improve as we use more prompts.

{\textbf{Observation1}}: \textit{Reducing the number of prompts does not linearly decrease the accuracy.}  
In addition to Fig. \ref{fig:intro:performance}, we provide further empirical evidence on the correlation between the number of prompts and the fine-tuning accuracy.
We evaluate the test accuracy of our approach on FGVC and VTAB-1k \cite{zhai2019large} tasks with varying the number of prompts. It is worth noting that each dataset requires a specific number of prompts to achieve optimal performance, as reported in \cite{jia2022visual}. We focus on datasets that require more than 10 prompts for both VPT-Shallow and VPT-Deep since using fewer than 10 prompts does not result in significant computational overhead.
We present the performance change according to the number of prompts in Table \ref{table:analysis:acc_vs_prompt}. Our analysis shows that for the majority of the datasets, decreasing the number of prompts by about 50\% does not result in a significant decline in performance. Additionally, most of the performance decrease occurred in the range of 10\% to 40\%, indicating that the relation between accuracy and the number of prompts is not linear.

{\textbf{Observation 2}}: \textit{Self-attention matrix is low-rank before/after adding prompts.}
The previous work \cite{wang2020linformer} shows that the self-attention matrix in ViT is low-rank. In a similar line of thought, we investigate the rank of the self-attention matrix when we add prompts.
In Fig. \ref{fig:anlysis:performance}, we compare the cumulative eigenvalue of the self-attention matrix $A$ \textit{without} prompts and \textit{with} prompts.
 Our results show that the self-attention matrix remains low-rank even when prompts are added to the self-attention matrix.
 Especially, for the Stanford Cars dataset, we add 200 prompts which is a large number of tokens than the original image tokens (\ie 197), but the cumulative eigenvalue trend does not change.
Overall, the results imply that only a few prompts affect the self-attention operation.

To understand why the number of prompts is not linearly correlated to the self-attention operation and the accuracy, we provide a mathematical analysis here.
We use the rank of the approximated low-rank matrix of the attention matrix as a surrogate metric to evaluate the impact of the prompt on the self-attention operation.
\begin{theorem}[Self-attention is low rank. Proved in \cite{wang2020linformer}]
 Let ${A} \in \mathbb{R}^{n \times n}$ be a self-attention matrix, and $v \in \mathbb{R}^{n}$ be a column vector of value matrix $V$. Then, there exists a low-rank matrix $\Tilde{A} \in \mathbb{R}^{n \times n}$ satisfying 
 \begin{equation}
     Pr(\|\Tilde{A}v^T - {A}v^T \| < \epsilon \|Av^T\|) > 1-o(1),
 \end{equation}
 where the rank of $\Tilde{A}$ is bounded, \ie, $rank(\Tilde{A}) = \Theta(log(n))$.
\label{theorem1}
\end{theorem}

\begin{proposition}
 For any low-rank matrices $\Tilde{A}_{n} \in \mathbb{R}^{n \times n}$ and $\Tilde{A}_{n+m} \in \mathbb{R}^{(n+m) \times (n+m)}$ satisfying $Pr(\|\Tilde{A}v^T - {A}v^T \| < \epsilon \|Av^T\|) > 1-o(1)$, we have 
\begin{equation}
    rank(\Tilde{A}_{n+m}) - rank(\Tilde{A}_n) = O(log(m)),
    \label{eq:bigo_logm}
\end{equation}
where $m$ is the number of prompts.
\end{proposition}
\begin{proof}

\noindent Based on Theorem \ref{theorem1}, given a bounded error  
$Pr(\|\Tilde{A}v^T - {A}v^T \| < \epsilon \|Av^T\|) > 1-o(1)$, the rank of $\Tilde{A}_{n}$ and $\Tilde{A}_{n+m}$ can be:
\begin{equation}
    \alpha log(n) \le rank(\Tilde{A}_{n}) \le \beta log(n),
\end{equation}
\begin{equation}
    \alpha log(n+m) \le  rank(\Tilde{A}_{n+m}) \le \beta log(n+m),
\end{equation}
where $\alpha$ and $\beta$ are the constants for the lower and upper bound respectively. Then, we have 
\begin{equation}
\footnotesize
    log\left(\frac{(n+m)^{\alpha}}{n^{\beta}}\right) \le  rank(\Tilde{A}_{n+m}) - rank(\Tilde{A}_{n}) \le log\left(\frac{(n+m)^{\beta}}{n^{\alpha}}\right).
\end{equation}
We obtain Eq. \ref{eq:bigo_logm} with respect to the variable $m$.
Additional details can be found in the Supplementary.
\end{proof}

Proposition \ref{theorem1} demonstrates that the increase of the rank of the low-rank self-attention matrix follows a logarithmic trend. 
As the logarithmic function is concave, the effect of adding new prompts on the attention operation diminishes as the number of prompts increases. For instance, increasing the number of prompts from $0$ to $50$ has a greater impact than increasing the number of prompts from $150$ to $200$.
This analysis is aligned with our \textbf{Observation 1} where reducing the number of prompts by approximately 50\% does not lead to a significant performance drop, but most of the performance drop exists in the $10\% \sim 40\%$ range.

\begin{figure}[t]
\begin{center}
\centering
\def\arraystretch{0.5}
\begin{tabular}{@{\hskip 0.01\linewidth}c@{\hskip 0.03\linewidth}c@{}c}
% \hspace{-4mm}
\includegraphics[width=0.48\linewidth]{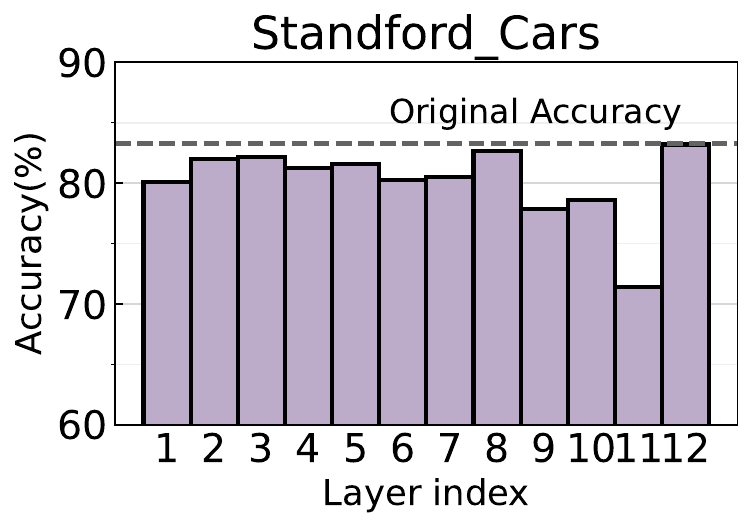} &
\includegraphics[width=0.48\linewidth]{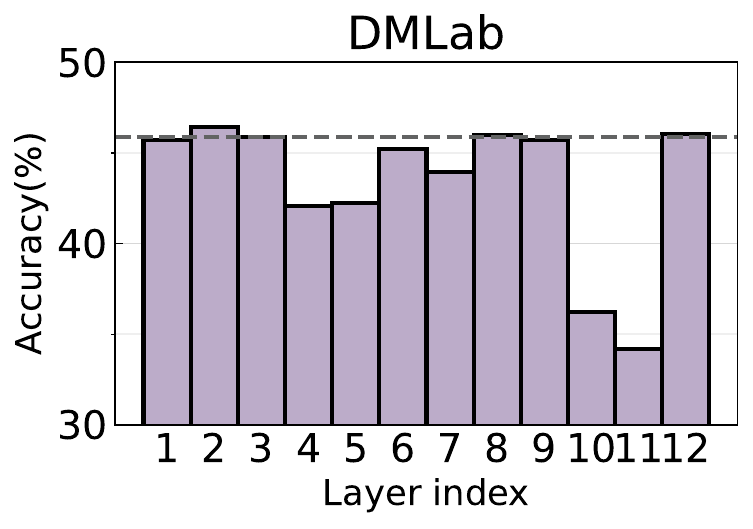} \\
\includegraphics[width=0.48\linewidth]{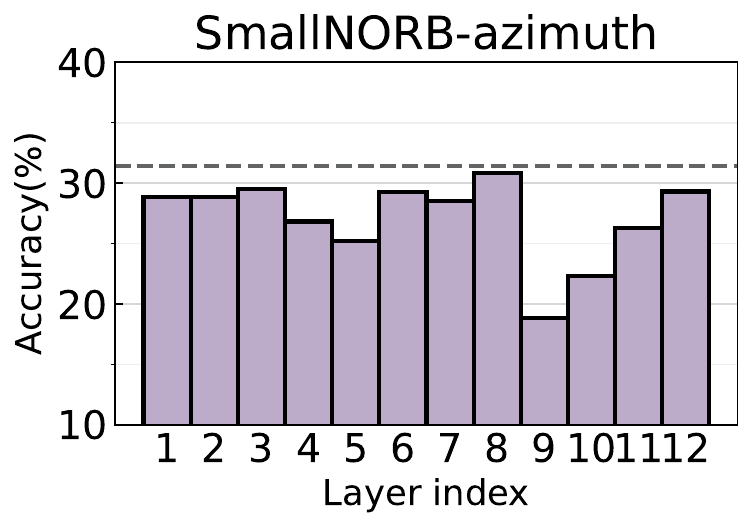} &
\includegraphics[width=0.48\linewidth]{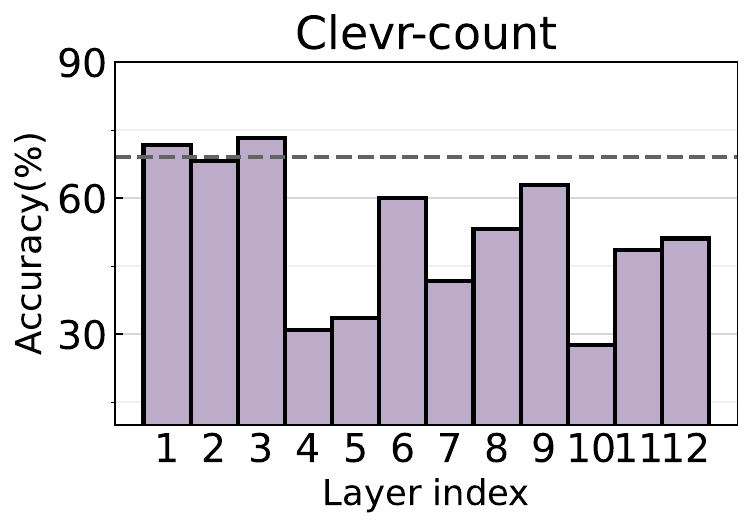} 
\end{tabular}
\end{center}
\vspace{-7mm}
\caption{ 
Accuracy changes by removing whole prompts in one layer. We report the original accuracy with a dotted line. Each dataset shows a different trend in accuracy degradation.}
 \vspace{-2mm}
\label{fig:analysis:layerwise_performance}
\end{figure}

\section{Prompt Condensation}

Although decreasing the number of prompts up to 50\% shows slight performance degradation, the performance drop is non-trivial in the small number of prompts regime. In Table \ref{table:analysis:acc_vs_prompt}, the major performance drop happens under 40\% of prompts on most datasets. To address this, we propose a technique called \textit{Prompt Condensation} (PC).

\noindent\textbf{Problem Statement.}
Our objective is to minimize the number of prompts while maintaining accuracy.
Let $P= \{p_1, p_2, ..., p_N\}$ be the set of prompts, and $P'$ be the condensed prompt set which has a smaller number of elements. Then our goal can be written as:
\begin{equation}
    \displaystyle\min_{P'} |\mathcal{L}(\theta, P)-\mathcal{L}(\theta, P')|,
\end{equation}
where $\mathcal{L}(\cdot)$ is the objective function of a task, $\theta$ is the model parameters.
At the same time, we also aim to minimize the number of prompts inside $P'$.
% and the cardinality of $P'$ is denoted as $|P'|$.

In designing our model for the Parameter Efficient Transfer Learning (PETL) scenario, we consider the following principles:
(1) Model parameters cannot be updated due to memory constraints. Therefore, only prompts can be trainable.
(2) Additional modules such as those proposed in \cite{rao2021dynamicvit,meng2022adavit,yin2022vit} cannot be utilized.
Given these constraints, most token sparsification methods are difficult to be applied in our case.
Instead, our method focuses on identifying important prompts and fine-tuning them without updating/adding any model parameters.

\noindent\textbf{Are all prompts equally important?} 
The important design choice for PC is whether to condense the same number of prompts for each layer or not.
To figure this out, we measure the accuracy change with respect to the prompts in each layer. 
We remove prompts in layer $l$ while other layers keep the same number of prompts.
As shown in Fig. \ref{fig:analysis:layerwise_performance}, we observe that prompts in different layers have varying contributions to the accuracy, and the trend           varies across different datasets. 
This observation leads us to leverage a global score across all layers, which is unlike the layer-wise score (\ie using the row similarity in self-attention) widely used in the previous work \cite{liang2022not,rao2021dynamicvit,fayyaz2022adaptive}.

\noindent\textbf{Prompt Scoring.}
We define the impact of prompt $p_i$ by computing the difference of the objective function from the fine-tuned VPT model.
\begin{equation}
    \|\Delta \mathcal{L}(\theta, p_i)\|_2 = \|\mathcal{L}(\theta, P)-\mathcal{L}(\theta, P'_i)\|_2,
    \label{eq:score_1}
\end{equation}
where $P'_i$ is the modified prompt set by zeroizing $p_i \in P$.
With Taylor approximation, we can approximate $\mathcal{L}(\theta, P')$ at $p_i=0$ as
\begin{equation}
\mathcal{L}(\theta, P_i') \approx \mathcal{L}(\theta, P) - \frac{d\mathcal{L(\theta)}}{dp_i}p_i.
    \label{eq:score_2}
\end{equation}
We only use the first-order term since the beyond second-order term requires huge memory storage.
If we substitute Eq. \ref{eq:score_2} to Eq. \ref{eq:score_1}, we obtain
\begin{equation}
    \|\Delta \mathcal{L}(\theta, p_i)\|_2 \approx \|\frac{d\mathcal{L}(\theta)}{dp_i}p_i\|_2.
        \label{eq:score_3}
\end{equation}
We average Eq. \ref{eq:score_3} across all data samples to compute the importance score.
\begin{equation}
   s_{p_i} = \frac{1}{|D|} \sum_{d \in D} \|\frac{d\mathcal{L}(\theta, d)}{dp_i}p_i\|_2,
        \label{eq:score_4}
\end{equation}
where $D$ is the input data set. Note that, calculating the importance score does not bring huge computational overhead since we only need to compute the backward gradient for the prompts.

{
Once we calculate the importance score for each prompt, we select the prompts with the highest k\% scores across all layers. This global prompt selection method inherently allocates the optimal number of prompts for each layer. On the other hand, with a local layer-wise prompt selection, we would enforce top k\% prompt selection uniformly across all layers that may 
inhibit the representation power within the model.
In our experiments, we show the global score provides better performance than the local layer-wise metrics.
}

Our approach is similar to filter pruning in CNNs \cite{lecun1989optimal,molchanov2016pruning} in the aspect of utilizing Taylor expansion. However, we have innovatively adapted this concept to the token level, presenting a fundamentally distinct granularity in pruning strategy. To our knowledge, our work is the first to employ gradient information directly for token pruning within the context of Vision Transformer (ViT) architectures. As a result, we believe our research paves the way for the potential application of existing channel pruning techniques to token pruning in ViTs.

\noindent\textbf{Overall Training Process.}
Algorithm 1 illustrates the overall process of Prompt Condensation.
We first train the original prompt set $P$ (Line 1). 
Then we compute the importance score of each prompt inside $P$ (Line 2). 
After that, we sort the importance score and select the prompts with the top $k\%$ of the importance score (Line 3). This provides the condensed prompt set $P'$. We discard the remaining $(100-k)\%$ of prompts.
Finally, the prompts within $P'$ are fine-tuned  (Line 4). For fine-tuning, we use less number of epochs $N_p$ compared to the original VPT training epochs $N_v$. We analyze the effect of $N_p$ in Section \ref{sec:exp_analysis}.
Note that the entire training process freezes weight parameters across the model except for the last classifier.

\begin{algorithm}[t]\small
        \caption{Prompt Condensation}
    %   \hspace*{\algorithmicindent}  Initialize the weightsW, scaling factorr, pruning ratiop, and the FIFO queueQwith lengthl
       \textbf{Input}: Training data $D$; Neural network $f(\cdot;P)$, Original prompt set $P$, Condensed prompt set $P'$, Condensation ratio $k \%$\\
      \textbf{Output}:  Trained  $f(\cdot; P')$
      \begin{algorithmic}[1]
        % \State{\textbf{begin}}
        %
        \State {Train $f(\cdot; P)$ on dataset $D$ for $N_{v}$ epochs}  \Comment{ VPT training}
        \State {Compute importance score $s_{p_i}$ across all prompts}   \Comment{Eq. \ref{eq:score_4}}
        \State {$P' \gets$ Select the prompts with top $k\%$ of $s_{p_i}$}
        \State {Finetune $f(\cdot; P')$ on dataset $D$ for $N_{p}$ epochs} 
      \end{algorithmic}
          \label{algorithm: overall}
          % \vspace{-2mm}
\end{algorithm}

\begin{figure*}[t]
\begin{center}
\centering
\def\arraystretch{0.5}
\begin{tabular}{@{\hskip 0.01\linewidth}c@{\hskip 0.01\linewidth}c@{\hskip 0.01\linewidth}c@{\hskip 0.01\linewidth}c@{}c}
% \hspace{-4mm}
\includegraphics[width=0.23\linewidth]{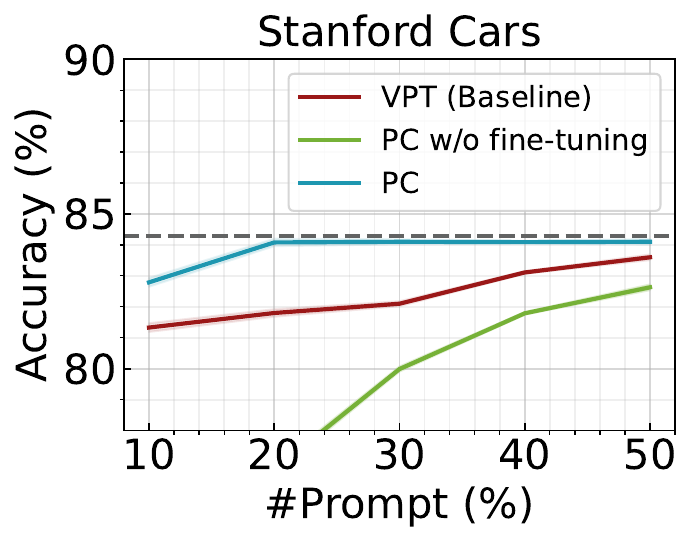} &
\includegraphics[width=0.23\linewidth]{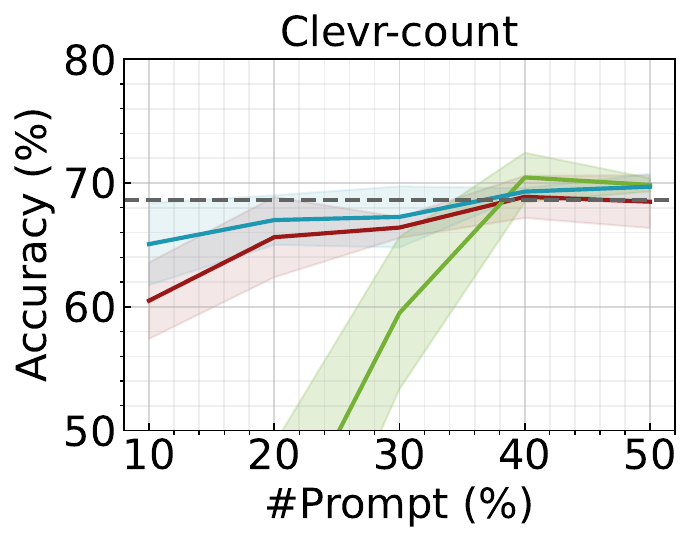}&
\includegraphics[width=0.23\linewidth]{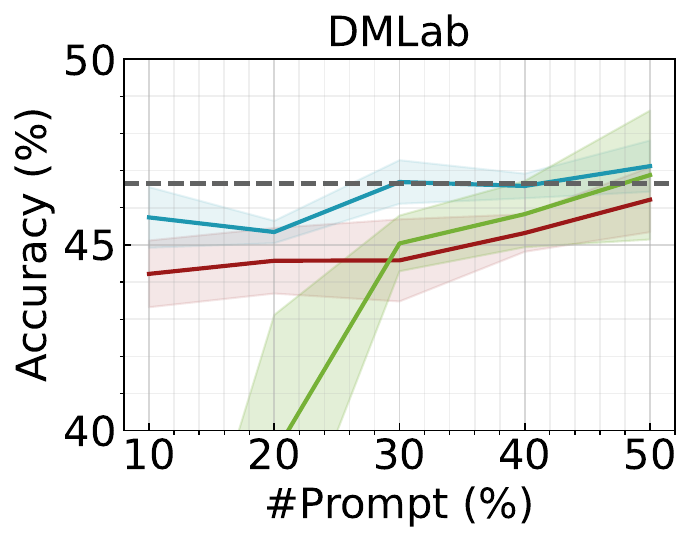}&
\includegraphics[width=0.23\linewidth]{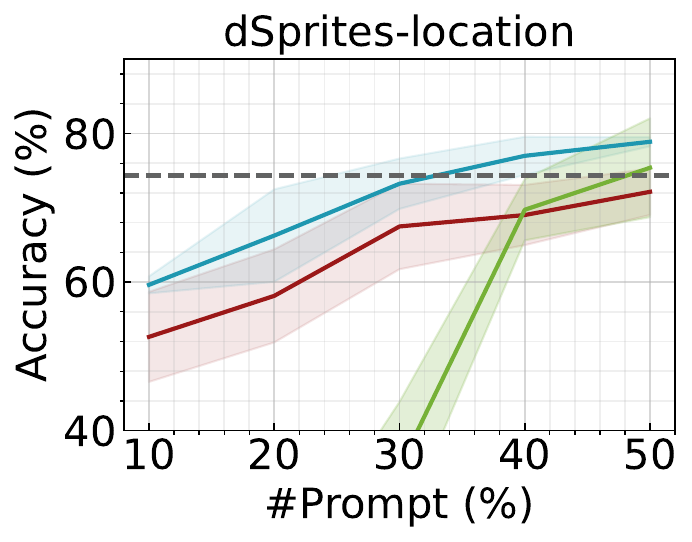} \\
\vspace{-1mm}
\includegraphics[width=0.23\linewidth]{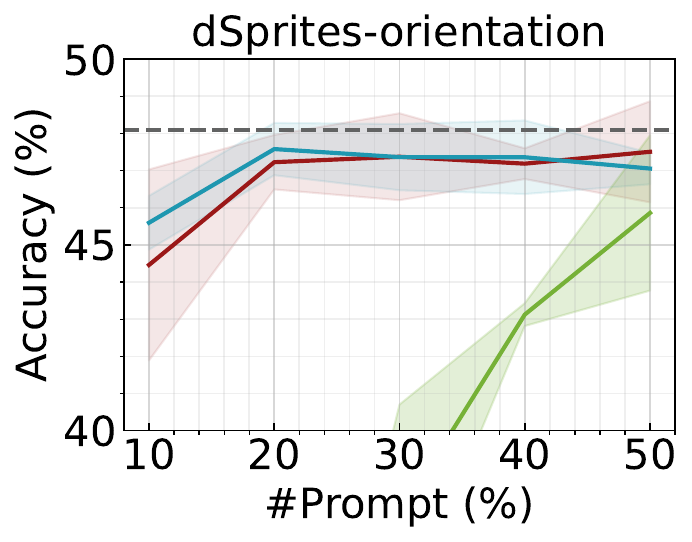}&
\includegraphics[width=0.23\linewidth]{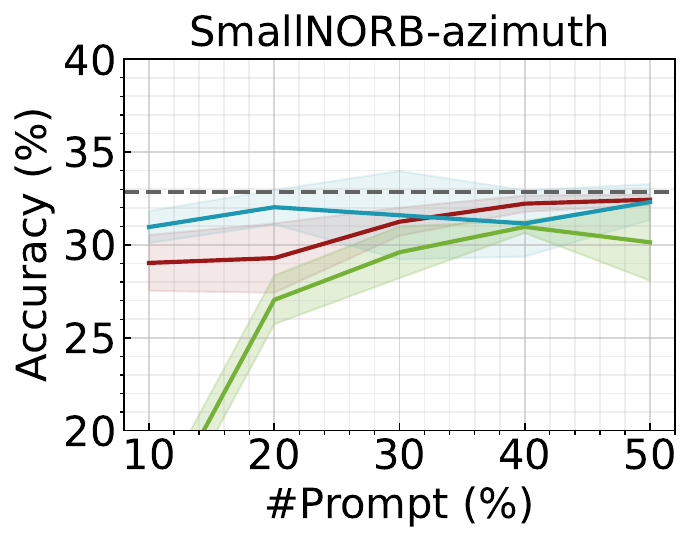}&
\includegraphics[width=0.23\linewidth]{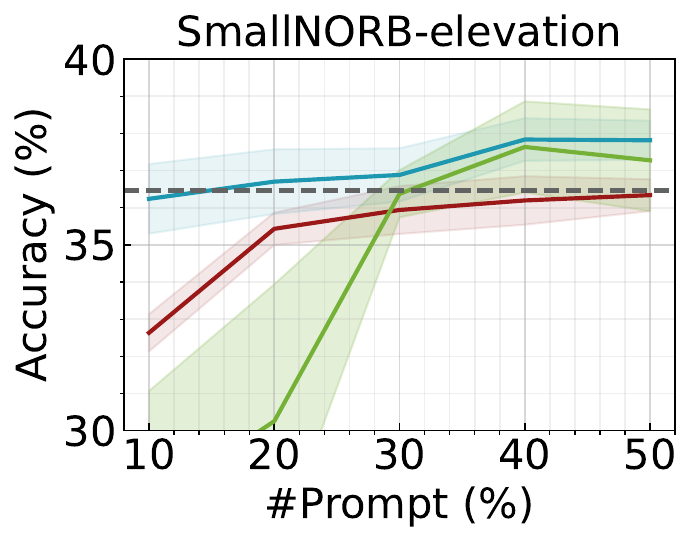} &
\includegraphics[width=0.23\linewidth]{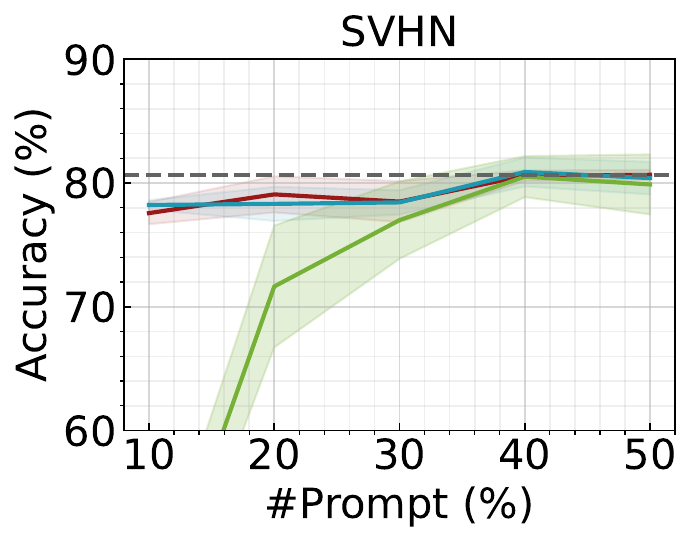} \\
% \hspace{-4mm}
% \hspace{-4mm}
\end{tabular}
\end{center}
\vspace{-6mm}
\caption{ 
 The test accuracy of VPT-Deep \cite{jia2022visual}, \textit{PC w/o fine-tuning}, and our proposed \textit{PC} with respect to the number of prompts. We use the ViT-B/16 backbone. A dotted line represents the accuracy with 100\% prompts.}
 \vspace{-2mm}
\label{fig:exp:pc_kd_results_deep}
\end{figure*}

\begin{figure*}[t]
\begin{center}
\centering
\def\arraystretch{0.5}
\begin{tabular}{@{\hskip 0.01\linewidth}c@{\hskip 0.01\linewidth}c@{\hskip 0.01\linewidth}c@{\hskip 0.01\linewidth}c@{}c}
% \hspace{-4mm}
\includegraphics[width=0.23\linewidth]{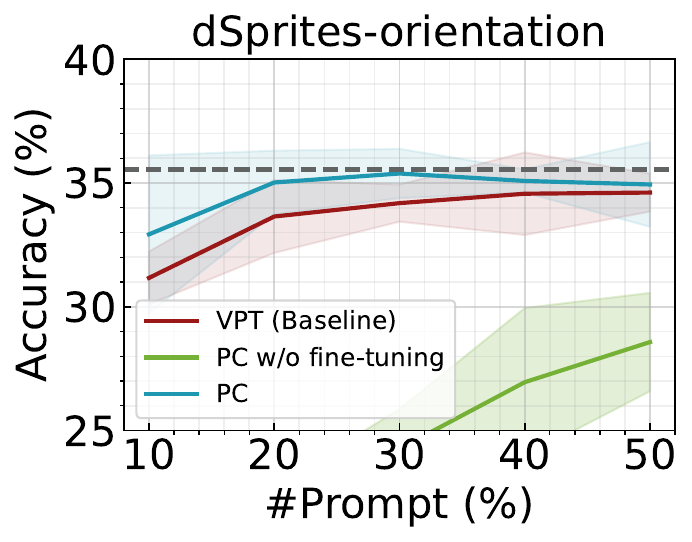} &
\includegraphics[width=0.23\linewidth]{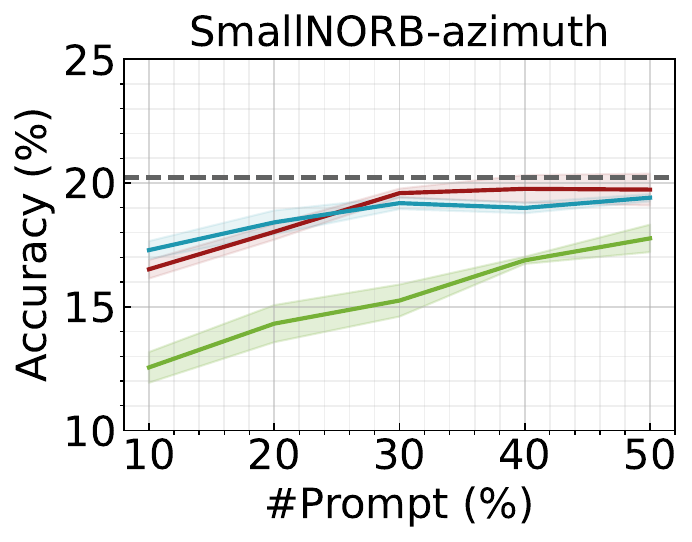}&
\includegraphics[width=0.23\linewidth]{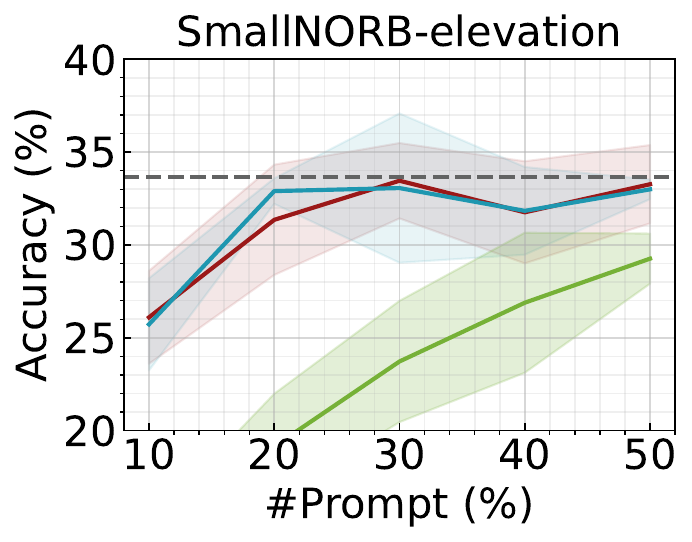} &
\includegraphics[width=0.23\linewidth]{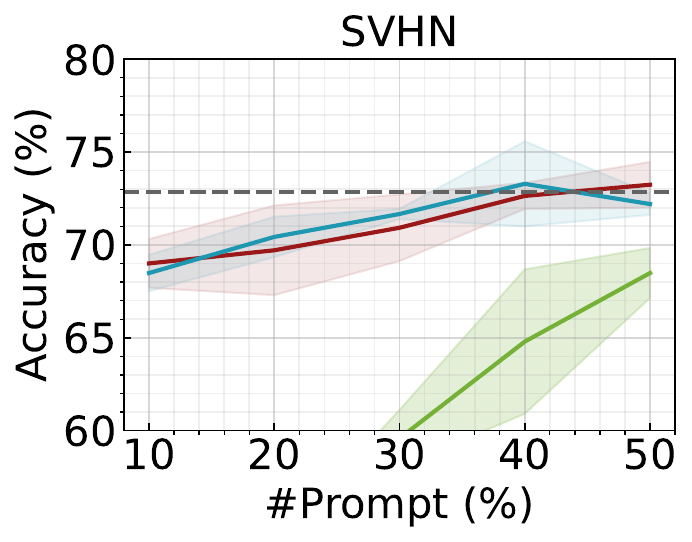}&
\end{tabular}
\end{center}
\vspace{-6mm}
\caption{ 
The test accuracy of VPT-Shallow \cite{jia2022visual}, \textit{PC w/o fine-tuning}, and our proposed \textit{PC} with respect to the number of prompts. We use the ViT-B/16 backbone.  A dotted line represents the accuracy with 100\% prompts.}
 \vspace{-2mm}
\label{fig:exp:pc_kd_results_shallow}
\end{figure*}

\begin{figure*}[t]
\begin{center}
\centering
\def\arraystretch{0.5}
\begin{tabular}{@{\hskip 0.0\linewidth}c@{\hskip 0.0\linewidth}c@{\hskip 0.0\linewidth}c@{\hskip 0.0\linewidth}c@{}c}
% \hspace{-4mm}
\includegraphics[width=0.20\linewidth]{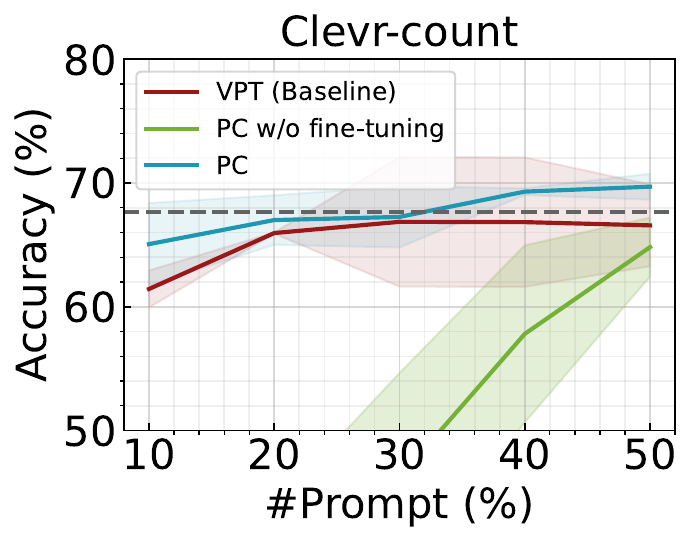} &
\includegraphics[width=0.20\linewidth]{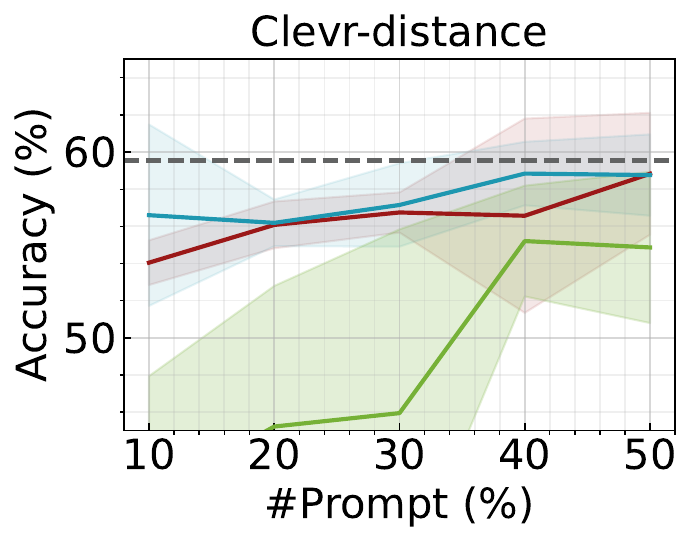}&
\includegraphics[width=0.20\linewidth]{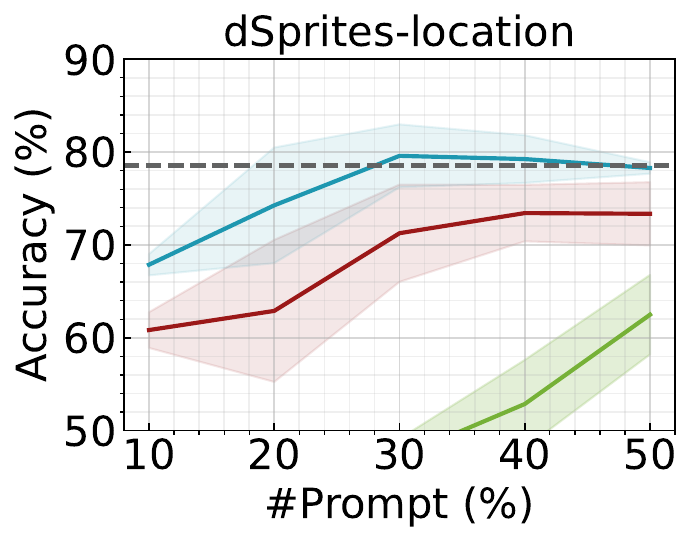}&
\includegraphics[width=0.20\linewidth]{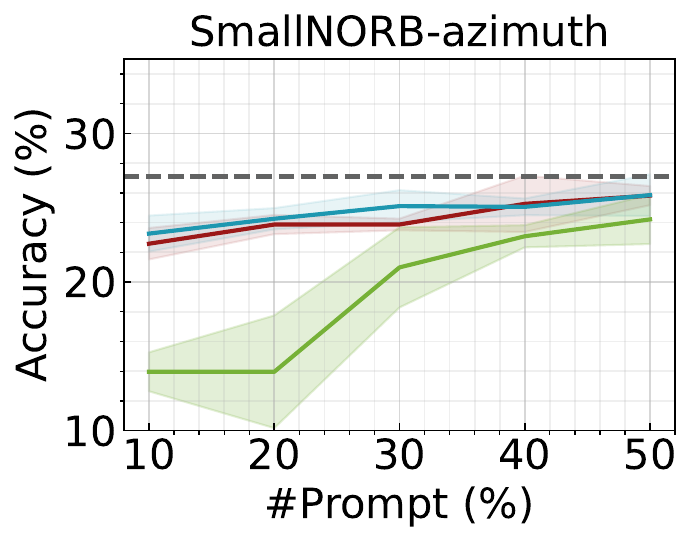} &
\includegraphics[width=0.20\linewidth]{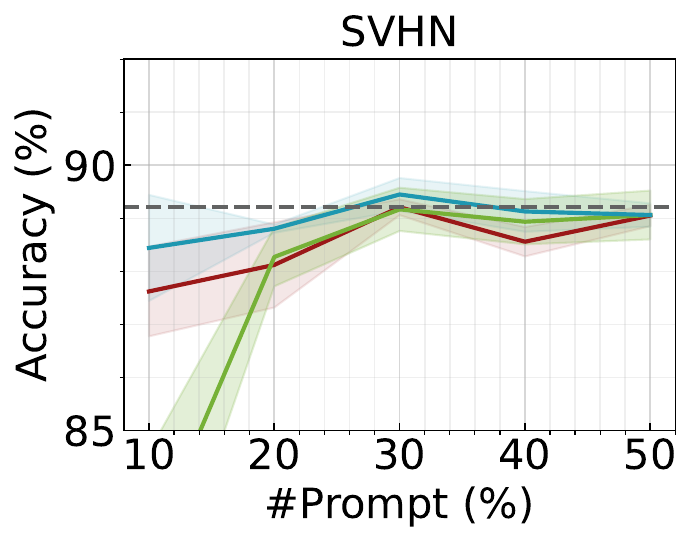} 
\end{tabular}
\end{center}
\vspace{-6mm}
\caption{
 The test accuracy of VPT-Deep \cite{jia2022visual}, \textit{PC w/o fine-tuning}, and our proposed \textit{PC} with respect to the number of prompts. We use the Swin-B backbone. A dotted line represents the accuracy with 100\% prompts.}
% The test accuracy of VPT-Deep \cite{jia2022visual}, Prompt Condensation without fine-tuning (\textit{PC w/o fine-tuning}) and Prompt Condensation (\textit{PC}) with respect to the number of prompts. Both \textit{PC} and \textit{PC w/o fine-tuning} use the Swin-B backbone. A dotted line represents the accuracy with 100\% prompts.}
 \vspace{-1mm}
\label{fig:exp:swin_pc_kd_results_deep}
\end{figure*}

\section{Experiments}

\subsection{Experiment Setting}

\noindent\textbf{Architecture.}
We conduct experiments using two transformer architectures pre-trained on ImageNet-22k, \ie, Vision Transformer (ViT-B/16) \cite{dosovitskiy2020image} and Swin Transformer (Swin-B) \cite{liu2021swin}.

\noindent\textbf{Dataset.}
We use the FGVC and VTAB-1k tasks as our datasets. FGVC consists of 5 datasets, including CUB-200-2011 \cite{wah2011caltech}, NABirds \cite{van2015building}, Oxford Flowers \cite{nilsback2008automated}, Stanford Dogs \cite{khosla2011novel}, and Stanford Cars \cite{gebru2017fine}. VTAB-1k \cite{zhai2019large} contains 19 datasets with various visual domains. Following previous work \cite{zhai2019large,jia2022visual}, we use the provided 800-200 split of the trainset for training and report the average accuracy score on tests within three runs. For both FGVC and VTAB-1k datasets, we select datasets that show a non-trivial (\ie $\ge 1\%$) accuracy drop with $10\%$ prompts compared to the original VPT. As a result, we have 8 datasets for ViT: \{Stanford Cars, Clevr-count, DMLab, dSprites-location, dSprites-orientation, smallNORB-azimuth, smallNORB-elevation, SVHN\}, and 5 datasets for Swin: \{Clevr-count, Clevr-distance, dSprites-location, smallNORB-azimuth, SVHN\}. The details of dataset selection are provided in the Supplementary.

We observed that non-trivial performance drops tend to occur in more challenging downstream tasks. To illustrate this, we calculated the mean and standard deviation of test accuracies across downstream tasks, separating them into those with non-trivial ($\ge 1\%$) and trivial ($< 1\%$) performance drops. These tasks were derived from the FGVC and VTAB-1k datasets.
Our results show that datasets with non-trivial accuracy drops exhibit an average accuracy of $58.91_{\pm 20.23}\%$, while those with trivial accuracy drops demonstrate a higher average accuracy of $81.96_{\pm 11.54}\%$.

\noindent\textbf{Hyperparameters.}
We follow the hyperparameters (\eg, weight decay, learning rate) reported in \cite{jia2022visual}. Each dataset had a different number of prompts, determined by the previous work \cite{jia2022visual} that reported the best-performing number of prompts for each dataset. During the prompt fine-tuning stage, we turn off dropout and used a $\times 0.1$ of the original VPT learning rate. For the prompt condensation, we retrain the selected prompts for 20 epochs, which is shorter than the original VPT training process. In Algorithm 1, we set the number of epochs $N_v$ for training VPT to 100, following the original paper \cite{jia2022visual}.

\subsection{Performance Comparison}
We first evaluate the effectiveness of Prompt Condensation (PC) with a limited number of prompts.
Specifically, we vary the number of prompts from $10\%$ to $50\%$, where the notation of $k\%$ denotes the use of $k\%$ of the number of prompts reported in \cite{jia2022visual}.
We compare the performance of PC with the following models:

\noindent {$\bullet$ VPT (baseline)}: We train the ImageNet pre-trained model with $k\%$ of original prompts. 

\noindent {$\bullet$ PC w/o fine-tuning}: From the trained VPT with $100\%$ of prompts, we compute the importance score (Eq. \ref{eq:score_4}) of each prompt and select top $k\%$ of prompts based on the score  and discard the remainder.

% \noindent {$\bullet$ PC w/ fine-tuning}: We follow Algorithm 1.

Fig. \ref{fig:exp:pc_kd_results_deep} and \ref{fig:exp:pc_kd_results_shallow} present a comparison of the performance of VPT-Deep and VPT-Shallow, respectively. From the results, we make the following observations:
(1) For VPT-Deep, PC maintains the performance with only $20\sim 30\%$ number of prompts, demonstrating its effectiveness compared to the naive VPT baseline. 
(2)  The performance gain achieved by applying PC to VPT-Shallow is comparatively lower than that of VPT-Deep. This can be attributed to VPT-Shallow having a smaller number of original prompts, which results in less room for performance improvement. At the same time, VPT-Deep yields higher performance than the VPT-Shallow model. Therefore, we focus on VPT-Deep in this paper.
(3) Interestingly, \textit{PC w/o fine-tuning} with VPT-Deep does not demonstrate a significant performance drop with $40\sim 50\%$ of the original prompts. This suggests that our prompt importance score accurately reflects the impact of each prompt on the overall accuracy.
(4) For the $10\sim 30\%$ regime, there is considerable performance degradation without fine-tuning. However, this can be fully recovered by fine-tuning prompts, demonstrating fine-tuning is an essential stage for PC. 
(5) The results of Swin also provide a similar trend as ViT, as shown in Fig. \ref{fig:exp:swin_pc_kd_results_deep}.

\subsection{Experimental Analysis}
\label{sec:exp_analysis}

\begin{table}[t]
   \centering
\small
\resizebox{0.48\textwidth}{!}{%
\begin{tabular}{l|l|ccccc}
\hlinewd{1pt}
% \hline
 Datasets & Methods & 10\% & 20\% & 30\% & 40\% & 50\%   \\
\hline
\hline
 \multirow{3}{*}{StanfordCars} &[CLS]-Sim & 81.67 & 82.63  & 83.35 & 83.95  & 84.08  \\
 & Ours-Local & 81.79 & 82.88  &  83.51 & 84.01  & 84.05 \\
& Ours-Global & \textbf{82.79}  & \textbf{84.08}  & \textbf{84.10}  & \textbf{84.09} & \textbf{84.10} \\
\hline
  \multirow{3}{*}{Clevr-count} &[CLS]-Sim & 57.25 & 63.02 & \textbf{67.51} & 67.40 & 67.92  \\
 & Ours-Local & 57.95 & \textbf{67.09}  & 66.04 & 66.58  &  66.75\\
& Ours-Global & \textbf{65.06} & {67.00}& {67.26} & \textbf{69.30} & \textbf{69.70}  \\
\hline
  \multirow{3}{*}{DMLab} &[CLS]-Sim & 45.44 & 45.02 & 46.25 & \textbf{46.81} & 46.74 \\
 & Ours-Local & 45.33 & 45.25  & 46.45 & 45.91  &  47.06\\
& Ours-Global & \textbf{45.74} &\textbf{45.34 }& \textbf{46.69} & {46.58} & \textbf{47.12}  \\
\hline \multirow{3}{*}{dSprites-loc} &[CLS]-Sim & 49.84 & 50.30 & 67.33 & 68.54 & 72.28  \\
 & Ours-Local & 57.95 & 60.26  & 63.26 & 71.30  &  72.87\\
& Ours-Global & \textbf{59.62} & \textbf{66.25} & \textbf{73.23} & \textbf{76.99}&\textbf{78.88}
\\
\hline \multirow{3}{*}{dSprites-ori} &[CLS]-Sim & 40.95 & 45.64 & 46.68 & 46.69 & 46.62  \\
 & Ours-Local & 44.09 & 45.08  & 46.71 & 46.33  &  46.76\\
& Ours-Global & \textbf{45.59} &\textbf{47.58} & \textbf{47.36} & \textbf{47.36 }& \textbf{47.05}  \\
\hline \multirow{3}{*}{SmallNORB-azi} &[CLS]-Sim & 28.93 & 29.48 & \textbf{32.52} & \textbf{31.87} & \textbf{32.33}  \\
 & Ours-Local & 30.29 & 30.60  & 31.24 & {31.80}  &  32.29\\
& Ours-Global & \textbf{30.96} & \textbf{32.03} & 31.59 & 31.15 & 32.31  \\
\hline \multirow{3}{*}{SmallNORB-ele} &[CLS]-Sim & 35.83 & 36.47 & \textbf{37.90} & 37.62 & 37.31  \\
 & Ours-Local & 35.66 & 36.22  & 36.04 & \textbf{38.41}  & \textbf{38.66}\\
& Ours-Global & \textbf{36.24} &\textbf{36.70} & 36.88 & 37.83 & 37.81  \\
\hline \multirow{3}{*}{SVHN} &[CLS]-Sim & 74.26 & 75.87 & 78.06 & 79.63 & 80.00  \\
 & Ours-Local & 76.75 & 78.06  & \textbf{79.18} & \textbf{79.67}  & 79.77\\
& Ours-Global & \textbf{78.22} &\textbf{78.31} & 78.43 & 80.88 & \textbf{80.39} \\
\hline
\hline
\multirow{3}{*}{Average} &[CLS]-Sim  &51.77 &	53.55&	57.45&	57.81&	58.41\\
 & Ours-Local  & 53.72	&55.68&	56.55&	58.00&	58.52\\
& Ours-Global  & \textbf{55.52} &\textbf{57.16} &\textbf{58.19}&	\textbf{59.27}& 	\textbf{59.67} \\
\hlinewd{1pt}
\end{tabular}%
}
\vspace{-1mm}
\caption{Test accuracy is evaluated using different prompt scoring techniques on VPT-Deep with ViT-B/16, with the best performance highlighted in \textbf{bold}.
}
\label{table:exp:prompt_scoring}
  \vspace{-1mm}
\end{table}

\begin{figure}[t]
\begin{center}
\centering
\def\arraystretch{0.5}
\begin{tabular}{@{\hskip -0.00\linewidth}c@{}c@{}c}
% \hspace{-4mm}
\includegraphics[width=0.9\linewidth]{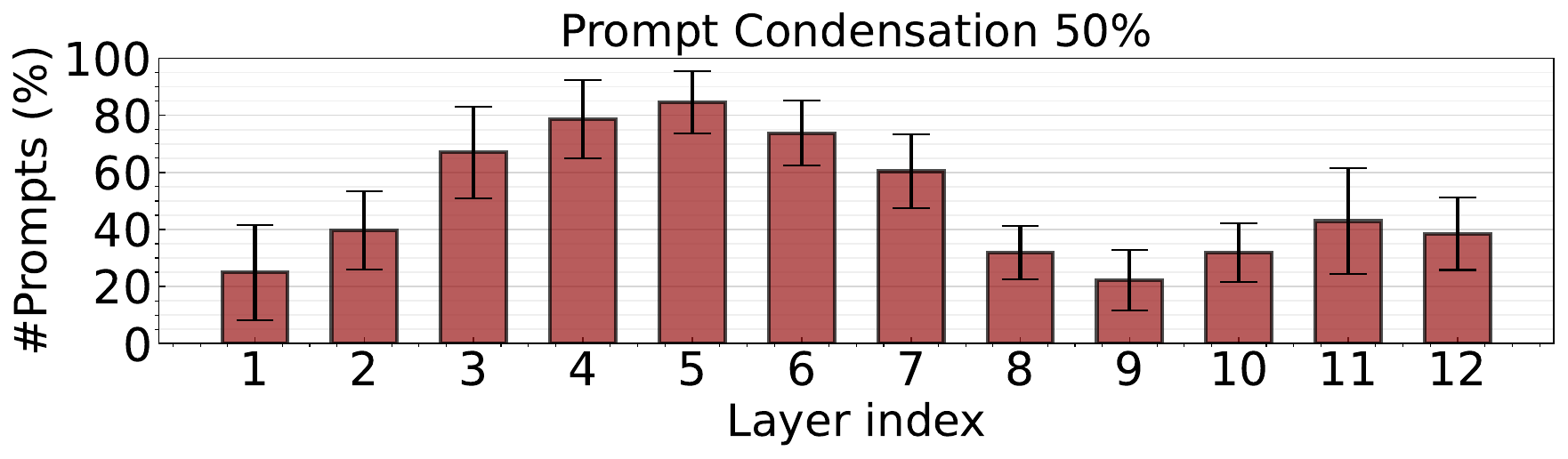}  \\
\includegraphics[width=0.9\linewidth]{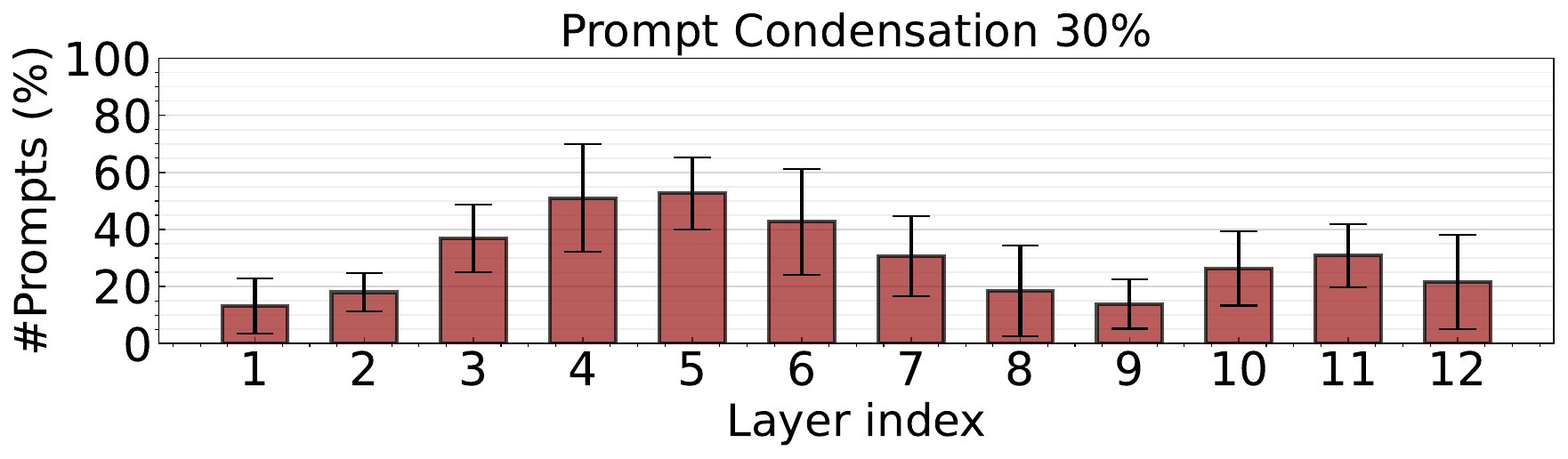}  \\
\includegraphics[width=0.9\linewidth]{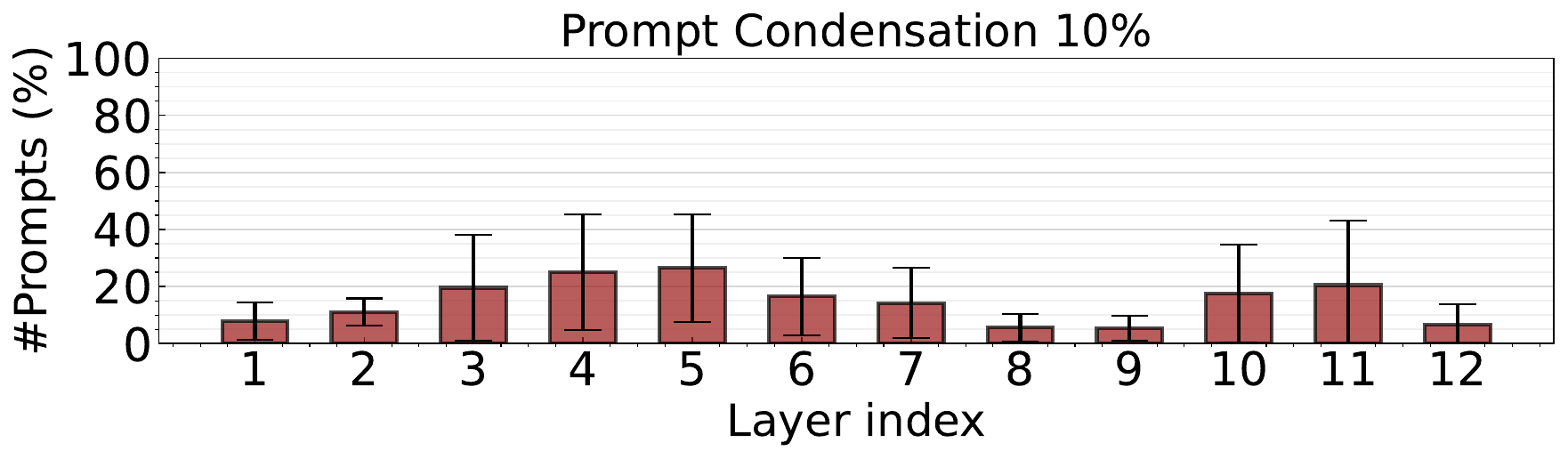} 
\end{tabular}
\end{center}
\vspace{-6mm}
\caption{ 
 Layer-wise prompt distribution with three different prompt condensation levels.}
 \vspace{-1mm}
\label{fig:anlysis:performance_layerwise}
\end{figure}

\noindent\textbf{Design Choice for Prompt Scoring.}
In our method, we compute the gradient to evaluate the importance of each prompt (Eq. \ref{eq:score_4}). Based on this score, we select the top-$k\%$ highest-scored prompts across all layers. To investigate the effectiveness of our prompt scoring technique, we compare it with several variants.

\noindent {$\bullet$ Global Prompt Condensation (ours-Global)}: Our proposed method where the top-$k\%$ highest-scored prompts are selected across all layers.

\noindent {$\bullet$ Local Prompt Condensation (ours-Local)}: Instead of considering whole layers, we select top-$k\%$ scored prompts in one layer. This approach ensures that the number of selected prompts is the same across all layers.

\noindent {$\bullet$ [CLS]-Sim}: We adopt the self-attention similarity between prompt tokens and a [CLS] token as a scoring technique inspired by 
 a line of previous works \cite{liang2022not,rao2021dynamicvit,fayyaz2022adaptive}. Here, we also select top-$k\%$ highest-scored prompts in one layer. 

In Table \ref{table:exp:prompt_scoring}, we present a comparison of the performance achieved using three different prompt scoring techniques. The results demonstrate that our proposed global scoring method, which considers the importance of prompts across all layers, outperforms the other two scoring techniques, particularly for lower percentages of PC (\eg 10\%). Therefore, our findings suggest that a global scoring metric is necessary for PC, given that the significance of prompts varies across different layers.

\begin{table*}[t]
   \centering
\small
\resizebox{0.9\textwidth}{!}{%
\begin{tabular}{c|l|ccccc|ccccc|c}
\hlinewd{1pt}
\multirow{2}{*}{GPU type}&{Dataset} &\multicolumn{5}{c|}{{Local Prompt Condensation}}&\multicolumn{5}{c|}{{Global  Prompt Condensation}}&\\
% \hline
 & (Unit: ms) & 10\% & 20\% & 30\% & 40\% & 50\%   & 10\% & 20\% & 30\% & 40\% & 50\% & 100\% \\
\hline
\hline
% \multicolumn{13}{c}{{Rtx2080ti}}\\
% \hline
% \multicolumn{13}{c}{{Rtx2080ti}}\\
 \multirow{3}{*}{Rtx5000} &Stanford Cars & 356 & 387 & 432 & 473 & 502  & 358 & 391 & 434 & 473 & 501 & 697\\
 &DMLab & 334 & 360 & 375 & 386 & 410  & 338 & 352 & 375 & 390 & 410 & 493\\
& SVHN & 329 &334 & 347 & 356 & 360   & 328 & 336 & 348 & 356 & 360 & 404\\
\hline
% \multicolumn{13}{c}{{V100}}\\
% \hline
 \multirow{3}{*}{V100} &Stanford Cars & 243 & 256 & 281 & 316 & 332 &  245 & 269 & 296 & 320 & 348 & 462\\
& DMLab &223 & 245 & 256 & 266 & 271 & 227 & 240 & 251 & 264 & 278 & 331 \\
 &SVHN & 218 & 223 & 227 & 245 & 249 & 218 & 225 & 229 & 248 & 250 & 269\\
\hline
% \multicolumn{13}{c}{{A100}}\\
% \hline
 \multirow{3}{*}{A100} & Stanford Cars & 55 & 59 & 70 & 74 & 80 & 53 & 59 & 65 & 74 & 80 & 118 \\
 &DMLab & 51 & 55 & 57 & 58 & 60   & 51 & 54 & 56 & 59 &63 & 80\\
 &SVHN & 50 & 52 & 53 & 54 & 56& 50 & 51 & 52 & 54 & 56 & 60 \\
% \hline
\hlinewd{1pt}
\end{tabular}%
}
\vspace{-1mm}
\caption{ Latency with respect to the number of prompts in VPT \cite{jia2022visual}.  We use ViT-B/16 as a baseline model.
% We also report relative computational cost compared to the case without prompts.
}
\label{table:exp:GPUtime}
  % \vspace{-1mm}
\end{table*}

\begin{table*}[t]
   \centering
\small
\resizebox{0.9\textwidth}{!}{%
\begin{tabular}{l|ccccc|ccccc|c}
\hlinewd{1pt}
{Dataset   } &\multicolumn{5}{c|}{{Local Prompt Condensation}}&\multicolumn{5}{c|}{{Global Prompt Condensation}}&\\
% \hline
(Unit: GFLOPs) & 10\% & 20\% & 30\% & 40\% & 50\%  & 10\% & 20\% & 30\% & 40\% & 50\% & 100\% \\
\hline
% \multicolumn{13}{c}{{Rtx2080ti}}\\
% \hline
% \multicolumn{13}{c}{{Rtx2080ti}}\\
Stanford Cars & 19.43 & 21.30 & 23.18 & 25.08 & 26.99  & 19.43 & 21.31 & 23.21 & 25.12 & 27.04 & 36.77\\
 DMLab & 18.50 & 19.43 & 20.36 & 21.30 & 22.24  & 18.49 & 19.43 & 20.36 & 21.30 & 22.25 & 26.99\\
SVHN & 18.04 & 18.50 & 18.97& 19.43 & 19.90  & 18.02 & 18.50 & 18.99 & 19.42 & 19.90 &  22.24\\
\hlinewd{1pt}
\end{tabular}%
}
\vspace{-1mm}
\caption{ FLOPs with respect to the number of prompts in VPT \cite{jia2022visual}.  We use ViT-B/16 as a baseline model.
% We also report relative computational cost compared to the case without prompts.
}
\label{table:exp:FLOPS}
  \vspace{-2mm}
\end{table*}

\noindent\textbf{{Layer-wise prompt distribution.}}
We present the visualization of the layer-wise prompt distribution for PC with different percentages of prompts (50\%, 30\%, and 10\%) in Fig. \ref{fig:anlysis:performance_layerwise}. The average number of prompts in each layer is computed across all datasets, and the mean and standard deviation are reported. The results indicate that prompts in the early layers have a minimal impact on the accuracy for most datasets. Furthermore, reducing the percentage of prompts leads to higher standard deviation, implying that the optimal number of prompts varies across datasets. Therefore, a global PC method is necessary to determine the optimal number of prompts at each layer across different datasets.

\noindent\textbf{{GPU Latency Analysis.}}
We analyze the practical latency time of VPT with PC on GPUs.
Theoretically, the complexity of self-attention operation increases quadratically as the input token length increases.
However, this may not hold true in practice due to factors such as hardware specifications \cite{jia2022visual,dosovitskiy2020image}.
To investigate the advantage of PC on GPU latency, we measure the GPU latency for 64 images on three different GPU environments:  Quadro RTX5000, V100, and A100.
We performed experiments on three datasets with different original numbers of prompts (StanfordCar, DMLab, and SVHN datasets originally had 200, 100, and 50 prompts, respectively).
In Table \ref{table:exp:GPUtime}, we observe that the proposed PC reduces GPU latency for all configurations.
As we expected, the effectiveness of PC is higher in the case with a larger number of prompts such as Stanford Cars.
Moreover, the global PC yields similar latency as that of the local PC. This further corroborates the use of the global importance score which achieves higher accuracy with negligible computational overhead.
We measure the FLOPs of VPT with PC in Table \ref{table:exp:FLOPS} to support our observation of the GPU latency, where the results show a similar trend.

\begin{figure}[t]
\begin{center}
\centering
\def\arraystretch{0.5}
\begin{tabular}{@{\hskip 0.01\linewidth}c@{\hskip 0.01\linewidth}c@{\hskip 0.01\linewidth}c@{\hskip 0.01\linewidth}c@{}c}
\hspace{-4mm}
\includegraphics[width=0.50\linewidth]{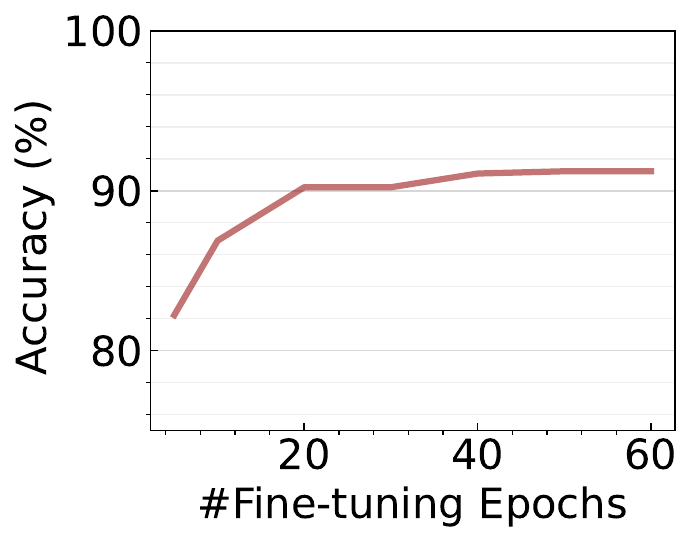} &
\includegraphics[width=0.52\linewidth]{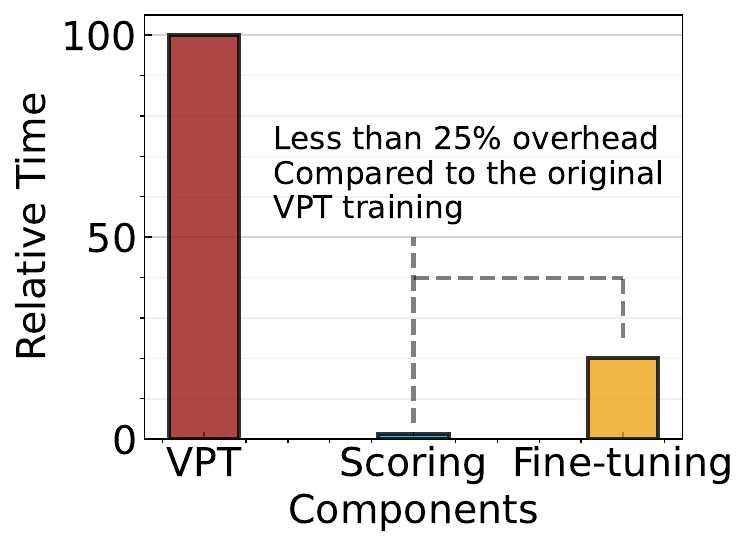}\\
\hspace{3mm}
(a) & (b)
\end{tabular}
\end{center}
\vspace{-6mm}
\caption{ 
 (a) The validation accuracy change with respect to the number of fine-tuning epochs. (b) Relative time of original VPT training, scoring, and prompt fine-tuning.}
 \vspace{-1mm}
\label{fig:exp:finetunig_analysis}
\end{figure}

\noindent\textbf{{Analysis on the Number of Fine-tuning Epochs.}}
One crucial hyperparameter in our method is the number of prompt fine-tuning epochs ($N_p$ in Algorithm 1). However, longer fine-tuning periods come with a higher computational cost, which is incompatible with on-device training scenarios. To determine the optimal number of fine-tuning epochs, we measure the average validation accuracy across all downstream datasets of VTAB-1K with 10\% prompts. As shown in Fig. \ref{fig:exp:finetunig_analysis}(a), the accuracy plateaus around epoch 20. Based on this observation, we set the number of fine-tuning epochs to 20 for all experiments.
In addition, Fig. \ref{fig:exp:finetunig_analysis}(b) illustrates the relative computational time between the original VPT training, prompt scoring, and prompt fine-tuning (line 1, line 2, and line 4 in Algorithm 1, respectively). The results demonstrate that our PC method (prompt scoring + fine-tuning) requires less than 25\% of the computational time needed for the original VPT training. These results indicate that our method is well-suited for on-device training scenarios.

\noindent\textbf{Practical Implementation of Prompt Condensation.}
In practical applications, it may not always be evident if there is a non-trivial performance drop with a small number of prompts. In such cases, we can use a relative computational cost metric (i.e., the ratio of [original image tokens] to [prompt + original image tokens]) to decide whether to apply Prompt Condensation (PC).
For instance, consider a scenario with 197 original tokens (196 + [CLS] token) and 100 prompt tokens. In this case, the addition of prompts results in a computational cost increase of $\frac{100}{197} = 50.76\%$. If the inclusion of prompts leads to an additional computational cost of $\ge K$\%, we can opt to implement PC. If not, it would be more beneficial to skip PC.

% \vspace{-1mm}

\section{Conclusion}
% \vspace{-1mm}
In this study, our aim is to investigate the influence of the number of prompts on VPT and its impact on both computational cost and fine-tuning performance. Our findings show that reducing the number of prompts by approximately 50\% does not significantly affect fine-tuned accuracy, with the majority of the performance drop occurring in the 10\% to 40\% range. Additionally, we demonstrated that increasing the number of prompts does not linearly enhance the maximum rank of approximated self-attention matrices. 
At the same time, we proposed Prompt Condensation (PC), a condensation technique that can effectively recover the performance degradation caused by using a small number of prompts. 
Overall, we hope our analysis and observations can provide insight to researchers in designing visual prompts.

{\small
\bibliographystyle{ieee_fullname}
\bibliography{egbib}
}

\end{document}